\documentclass[11pt]{article}
\usepackage{amsmath, amssymb, amsthm, fullpage, mathrsfs, verbatim, enumerate, marvosym, url, csquotes}
\usepackage[pdftex]{graphicx}
\usepackage{color}

\newtheorem{theorem}{Theorem}[section] 

\newtheorem{proposition}[theorem]{Proposition}
\newtheorem{definition}[theorem]{Definition}

\begin{document}
\title{Equalizing Financial Impact in Supervised Learning}
\author{Govind Ramnarayan\footnote{Email: govind@mit.edu. Supported by Elchanan Mossel's NSF award CCF 1665252 and NSF award DMS-1737944.}\\}
\maketitle
\begin{abstract}

Notions of ``fair classification'' that have arisen in computer science generally revolve around equalizing certain statistics across protected groups. This approach has been criticized as ignoring societal issues, including how errors can hurt certain groups disproportionately. We pose a modification of one of the fairness criteria from Hardt, Price, and Srebro [NIPS, 2016] that makes a small step towards addressing this issue in the case of financial decisions like giving loans. We call this new notion ``equalized financial impact.''

\end{abstract}

\section{Introduction}
\label{sec:abstract}

Machine learning is revolutionizing the way we interact with the world. Popular websites use algorithms to analyze user data and recommend videos, customize social media feeds, and optimize advertisements. Unsurprisingly, machine learning is taking a large role in making decisions about human beings, ranging from credit to parole decisions, and is likely to be more and more widely used in the future. It is not hard to imagine that, even in cases where the final decisions are made by people, they will be doing so with advice from algorithms that make inferences from patterns in petabytes of data.

 Some proponents of machine learning have suggested that not only are these algorithms able to leverage the increasing amount of data we have access to, but also that they might be able to make these decisions more fairly, as they seem to not be subject to human biases. There is some truth to these claims. For example, an algorithm to judge whether criminal defendants will recidivate can be used nationwide (or potentially even world-wide). This can alleviate biases that may be local. For instance, a judge in a certain county could be heavily biased against African-American defendants. By replacing her judgments with judgments generated by an algorithm, we can prevent this bias from affecting defendants in this county. Such positive effects could be achieved even by requiring that the judge agrees with the algorithm in 95\% of cases. However, the widespread use of machine learning also poses a major worry. If algorithms themselves can be biased, for some definition of this, then widespread usage may be institutionalizing this bias. Recently, many researchers have raised significant concerns that these algorithms can be biased in problematic ways.

One notable example is the COMPAS system, which has been used to score criminal defendants and judge whether or not they are likely to recidivate. These scores are then used in parole considerations. Thus, ensuring that these scores do not discriminate against certain demographic groups is of utmost importance. In their study of the COMPAS system, Larson, Mattu, Kirchner, and Angwin~\cite{LMKA16} looked at 10,000 criminal defendants and compared their predicted recidivism rates with their actual recidivism rates over a two year period. Their findings showed that African-American defendants were scored more harshly than their Caucasian peers:

``\ldots black defendants who did not recidivate over a two-year period were nearly twice as likely to be misclassified as higher risk compared to their white counterparts (45 percent vs. 23 percent)\ldots white defendants who re-offended within the next two years were mistakenly labeled low risk almost twice as often as black re-offenders (48 percent vs. 28 percent)\ldots'' (from~\cite{LMKA16})

This is evidence that automated classifiers can exhibit significantly different behavior on different races, despite the fact that they may not seem to be subject to human biases. The above behavior is at the very least worrying, and suggests that the algorithm should be modified such that this behavior does not occur.

It is tempting to think that we can avoid such discrimination by not including group membership in our data. This sort of ``fairness through unawareness'' may seem attractive, as it might seem impossible for a classifier to discriminate against a protected group if it does not even know who is in that group. However, this is quite well-known to be false. Membership in the protected group is likely to be correlated with other data that the classifier is given. The fact that ``fairness through unawareness'' like we have described is basically unachievable has been brought up by many authors (see Dwork et al.~\cite{DHPRS11} for a further discussion of this). Dwork et al. advocate that we instead strive for ``fairness through awareness'' - by making our algorithms aware of protected groups, they can explicitly prevent certain biases from occurring. We also work in this framework.

\subsection{Our Contribution}
We propose the idea of considering how mistakes affect different groups, and factoring this into the existing ideas of group fairness. Namely, if group A is harmed more by say, false negatives, than group B, then a classifier should aim to make fewer false negatives on group A than group B. This idea is fairly broad, and has independently been raised by other authors. One example is Altman et al.~\cite{AVW18}, who make a convincing argument that such a ``harm-based'' approach to fairness is preferable to equalizing error rates without accounting for how mistakes can harm different groups differently. In their argument, they appeal to well-studied ethical scenarios like the effects of (false) incarcerations on African-American communities to illustrate how errors in classification can harm certain communities disproportionately. 

We raise a similar idea about equalizing the affect of mistakes across groups, and focus on equalizing false negative rates for loans. We think that this case is particularly important for three reasons. First, the decreasing marginal utility of wealth gives a clear reason why denying a poor, but qualified, loan applicant is inherently worse than denying a rich, qualified loan applicant from the point of view of the applicant. Second, researchers have empirically studied how utility depends on wealth in various cases (see e.g.~\cite{BT97}), which suggests that we can estimate approximately \emph{how much} more harm a false negative on a poor loan applicant creates than a false negative on a rich applicant. Finally, equalizing impact for loans has the indirect effect of favoring historically disadvantaged groups, as people in these groups often have lower incomes. Hence, equalizing impact moves in the direction of notions of fairness suggested by theories of corrective justice~\cite{Binns17}~\cite{Parfit97}.

In this paper, we describe how to equalize impact for false negatives for loans, given a marginal utility function. Our main contribution is describing a natural way to factor individual utility considerations into the group fairness framework of Hardt, Price, and Srebro~\cite{HPS16}. For the loans case, this is basically done by grouping together people with similar incomes, and weighting mistakes on each group accordingly (Section \ref{sec:main-section}). However, the idea of equalizing expected impact of mistakes can be applied much more broadly, and indeed researchers (including Altman et al.~\cite{AVW18}) have given strong ethical reasons for believing that this philosophy should be applied in every scenario, for false positive and false negative rates. To this end, we note that the method we describe for equalizing financial impact of false negatives can be generalized to equalizing false positive rates or false negative rates for any scenario. The requirement for doing so is that we need a good approximation of how much harm a mistake (either false positive or false negative) does to a person based on their attributes. If this approximation depends strongly on many attributes, then the classifier needs a large number of labelled samples to enforce this condition\footnote{Roughly exponential in the number of attributes.} (see Section \ref{sec:discussion} for slightly expanded discussion of this).

In Section \ref{sec:notation}, we outline the notation we will use for the rest of the paper. In Section \ref{sec:preliminaries}, we describe how some computer scientists have been addressing these issues by enforcing certain statistical constraints on classifiers, in particular the work of Hardt, Price, and Srebro~\cite{HPS16} which describes how to modify classifiers such that they have equal false positive and false negative rates across groups. In Section \ref{sec:main-section}, we will further describe why believe this sort of criterion can be unfair, and suggest a way to modify the framework of~\cite{HPS16} to support this modification. 
\section{Notation}
\label{sec:notation}
We will consider the setting where supervised learning is used to learn (fair) classifiers. A classifier will be a function from examples (in our setting, descriptions of people) to binary-valued decisions\footnote{These decisions need not be binary-valued in general, but we will focus on the binary case.}, and we would like these decisions to be accurate. In supervised learning, a learning algorithm uses labeled examples to learn an accurate classifier. We define some notation relating to supervised learning that we will use for the rest of the paper. We refer the reader who is interested in a more detailed introduction to supervised learning in the context of fairness to the excellent survey of Barocas and Selbst~\cite{BS16}.
\begin{itemize}
\item Let $X$ denote the set of features.
\item Let $A \subseteq X$ denote the set of protected features. For a fixed $a \in A$, we will often refer to the set of people with $A=a$ as members of the protected group $a$. Note that not all protected groups are necessarily disadvantaged. Protected groups are simply instantiations of protected attributes. For example, if one of the protected features is an ethnicity, it will be prudent to allow one of the groups to be the majority ethnicity for our definitions.
\item Let $Y$ denote a binary-valued \emph{ground truth} - that is, if we are considering the problem of giving out loans, someone who actually pays back a series of loans would have $Y=1$.\footnote{In real situations, $Y$ may not be determined at the time the decision is made. However, it is useful from an algorithmic perspective to view it as fixed at the time of decision.}
\item Let $\hat{Y}$ denote the binary output of the classifier - in the case of loans, whether or not the classifier decides to give them a loan.
\item Given a group $a \in A$, we will denote the \emph{base rate} of $a$ as $\Pr(Y = 1 | A = a)$, where the probability is taken over all the people we will classify.
\item The \emph{loss} of classifier will be gauged according to a loss function $\ell:\hat{Y} \times Y \to \mathbb{R}$. That is, the expected loss of a classifier $\hat{Y}$ will be $\mathbb{E}[\ell(\hat{Y}, Y)]$. 
\end{itemize}
Finally, for random variables $A$ and $B$, we say that $A \perp B$ if $A$ and $B$ are probabilistically independent.
\section{Preliminaries: Statistical Parity and Equalized Opportunity}
\label{sec:preliminaries}
\subsection{Overview}
Recall that a supervised learning algorithm optimizes some classifier with respect to labeled training examples. However, this trained classifer may exhibit statistical biases towards certain groups. 
Many works explicitly enforce statistical constraints which prohibit the types of biases we are worried about. The following are examples of these types of statistical constraints:
\begin{enumerate}
\item The fraction of male defendants classified as ``low risk'' should equal the fraction of female defendants classified as ``low risk.'' 
\item The fraction of African-American defendants who are categorized as ``low risk'' but recidivate should be equal to the fraction of Caucasian defendants who are classified as ``low risk'' but recidivate. Similarly, the fraction of African-American defendants who are categorized as ``high risk'' but do not recidivate should be equal to the fraction of Caucasian defendants who are classified as ``high risk'' but do not recidivate.
\end{enumerate}
These two types of statistical constraints are fundamentally different. The first one stipulates that the fraction of people in one protected group classified positively (in this case, classified as ``high risk'') is equal to the fraction of people in another protected group classified positively. Note that this makes no appeal to the true base rates within these groups. If, for example, men were more likely to recidivate than women, this notion of statistical fairness does not care: it will still constrain the classifier to give the ``high risk'' classification to men and women equally. This notion is called \emph{statistical parity}, and is discussed in~\cite{DHPRS11}. We will briefly talk about it as relevant to the motivation of our definition; however, beyond motivation, it bears little relevance to our result.

The second notion stipulates that false positive rates and false negative rates should be equal across protected groups, where here we view a false positive as a person who is not going to recidivate receiving a ``high risk'' classification, and a false negative as a person who is going to recidivate receiving a ``low risk'' classification. Equivalently, this stipulates that protected attributes (in this case, race) are allowed to influence the rate of positive classifications (which is not true if we are enforcing statistical parity), but only to the extent that they are actually correlated with the true outcome (in this case, whether or not someone actually recidivates). This is known as \emph{equalized odds}, and appears in~\cite{HPS16}. 

\subsection{Statistical Parity and Justifying Affirmative Action}
\label{sec:stat-parity}
Statistical parity is a very natural idea, and was explicitly addressed in the seminal paper of Dwork et al.~\cite{DHPRS11}. Formally, this can be defined as follows. Fix any protected features $a_1, a_2 \in A$. Then statistical parity mandates that 
\[ \Pr[\hat{Y} = 1 | A = a_1] = \Pr[\hat{Y} = 1 | A = a_2] \]
where the probability is taken over all the people we will classify. In practice, we will be okay with guaranteeing approximate equality, up to an addition of a small constant $\varepsilon > 0$. Note that this condition implies that the classification $\hat{Y}$ is independent of the protected features $A$ (we will use the notation $\hat{Y} \perp A$ to denote probabilistic independence). Unfortunately, this comes with some major downsides. The main one is that, if the base rates are unequal, i.e. $\Pr(Y = 1 | A = a_1) \neq \Pr(Y=1 | A = a_2)$, then this stipulation prevents the ground truth from being permissible as a classifier; that is, it forces $\hat{Y} \neq Y$. In this case, the ground truth $Y$ is actually correlated with the protected features $A$, but statistical parity forces us to choose a classifier such that $\hat{Y} \perp A$. Furthermore, enforcing statistical parity might encourage the classifier to give out positive and negative classifications to people who do not deserve them, just to equalize the rates of classification across groups. While this may enforce group fairness, this type of behavior is blatantly unfair to the unlucky individuals in the group who are singled out for random treatment. Dwork et al. bring up the point that statistical parity can be very \emph{unfair} to individuals, but note that the notion of fairness that they endorse can sometimes be close to statistical parity, which yields a form of justifiable affirmative action. Their notion of fairness stems from enforcing that similar individuals are treated similarly by the classifier, for a problem-specific notion of similarity. We will not discuss it more in this document.

Statistical parity is often viewed as suboptimal for the above reasons. However, in light of the arguments made by Binns~\cite{Binns17} and inspired by Parfit~\cite{Parfit97}, we believe statistical parity should not be ignored due to its miscomings in certain scenarios. The scenarios where statistical parity could be bad are ones in which the base rates of different groups are unequal. Some of these examples are brought up in the paper of Dwork et al. But just because it is bad in some cases, this does not mean it is bad in all cases. In cases where this difference in base rates can be explained by historical mistreatment, perhaps current base rates are not the best guide for classification, and statistical parity might provide something that is closer to fair. 

However, statistical parity still suffers the problem that naively forcing a classifier to satisfy statistical parity without giving it enough information can lead to unfair treatment. Hence, rather than viewing statistical parity alone as a means of affirmative action, we think it is more advisable to incorporate harms caused by historical injustices into other metrics for fairness. The motivation for this is similar to the motivation for ensuring fairness through awareness; rather than blindly enforcing statistical parity without knowing why the true base rates of different groups are unequal, it makes more sense to understand \emph{why} base rates are unequal, and, to the extent that it is possible, explicitly factor those reasons into the constraints we impose on our classifier. Even if these types of considerations justify statistical parity in some cases, they will not justify it in others. Building up a theory of algorithmic fairness that factors in historical mistreatment is a great goal, but an ambitious one. In the meanwhile, we think it is advisable to consider statistical parity in certain cases where it can be justified by additional means, much like what is suggested by Dwork et al.~\cite{DHPRS11}.

\subsection{Equalized Odds and Opportunity}
\label{sec:equalized-odds}
Now we discuss the second notion of fairness, called equalized odds, which appears in the work of Hardt, Price, and Srebro~\cite{HPS16}. The motivation of equalized odds is that, unlike statistical parity, it does not prevent the ground truth from being a classifier in the case where base rates are unequal. Equalized odds enforces the conditions that the false positive rate and the true positive rate are the same across all groups. For any $a_1, a_2 \in A$, consider the following possible constraints on a classifier $\hat{Y}$:
\begin{equation}
\label{eq:false-positive-equal}
\Pr[\hat{Y}=1 | Y=0, A=a_1] = \Pr[\hat{Y}=1 | Y=0, A = a_2] \hspace*{2cm} \text{    (False positive rates equal)}
\end{equation}
\begin{equation}
\label{eq:true-positive-equal}
\Pr[\hat{Y}=1 | Y=1, A=a_1] = \Pr[\hat{Y}=1 | Y=1, A = a_2] \hspace*{2cm} \text{    (True positive rates equal)}
\end{equation}
Note that since $\Pr[\hat{Y}=0 | Y=1, A = a_1] = 1 - \Pr[\hat{Y}=1 | Y=1, A = a_1]$, Equation \ref{eq:true-positive-equal} is equivalent to stipulating that false negative rates are equal:
\begin{equation}
\label{eq:false-negative-equal}
\Pr[\hat{Y}=0 | Y=1, A=a_1] = \Pr[\hat{Y}=0 | Y=1, A = a_2] \hspace*{2cm} \text{    (False negative rates equal)}
\end{equation}
Equalized odds stipulates that false positive and false negative rates should be equal across groups.

\begin{definition}[Equalized Odds~\cite{HPS16}]
\label{def:equal-odds}
A classifier satisfies \emph{equalized odds} if Equation \ref{eq:false-positive-equal} and Equation \ref{eq:false-negative-equal} hold for any $a_1, a_2 \in A$.
\end{definition}

The combination of these two conditions is equivalent to ensuring that $(\hat{Y} \perp A) | Y$ -- or, in words, that the classification $\hat{Y}$ is conditionally independent of the protected attributes given the ground truth $Y$. As already mentioned, the main advantage of this definition over statistical parity is that this does constraint does not rule out the ground truth classifier $Y$, as clearly, $(Y \perp A) | Y$. This can also be seen directly from the equations above, as when the classifier is the same as the ground truth, we get that the false positive rates (resp. false negative rates) for every group are 0, and therefore trivially equal. In words, equalized odds permits a classifier to use the protected attribute $A$ in classification to the extent that it is \emph{meaningful}. If $A$ is highly correlated with $Y$, or in other words, if base rates across groups are very different, then the classifier $\hat{Y}$ is allowed to depend strongly on $A$. However, if $A$ is not very correlated with $Y$, then $\hat{Y}$ is not allowed to depend strongly on $A$. It may seem that there is no reason that an accurate classifier would be correlated with $A$ in the case where $Y$ is not very correlated with $A$. However, in practice, we do see instances of this type of unnecessary correlation, like in COMPAS, where false positive and false negative rates differ greatly between races. If the only condition we want from a classifier is that it maximizes its accuracy, a classifier might choose to correlate very highly with membership in a protected group in order to improve accuracy and/or generalization slightly. 

The authors of~\cite{HPS16} also pose a relaxed version of equalized odds, which they call \emph{equalized opportunity}. This only enforces that true positive rates (or equivalently, false negative rates) are equal.
\begin{definition}[Equalized Opportunity~\cite{HPS16}]
\label{def:equal-opp}
A classifier satisfies \emph{equalized opportunity} if Equation \ref{eq:false-negative-equal} holds for any $a_1, a_2 \in A$.
\end{definition}

 While this condition does not have the same clean interpretation as equalized odds in terms of conditional independence, it is an easier condition for a classifier to satisfy, and hence permits for a higher accuracy classifier. In cases where we are more worried about false negatives than false positives from the context of fairness, equalized opportunity can provide a meaningful fairness guarantee while still being easy to satisfy. Equalized opportunity is the inspiration for the condition we suggest in Section \ref{sec:main-section}. 

\section{Equalized Financial Impact}
\label{sec:main-section}
\subsection{Overview and Justification}
In this section, we argue that even without considering historical context for these disadvantages, it may be admissible, even necessary, to allow for smaller false negative rates on disadvantaged groups compared to other groups that are not disadvantaged. Our goal is to encourage researchers to revisit existing fairness notions by considering the \emph{expected impact} of incorrect decisions rather than just the rate of incorrect decisions. 
 
 Equalizing expected impact and rate could be the same in many scenarios where mistakes impact all protected groups equally. In the remainder of this paper, we will focus on the impact of false negatives for the problem of giving loans, though our argument applies without loss of generality to most financial decision problems, and can be naturally extended to more general problems as discussed in the Introduction and in Section \ref{sec:discussion}. To define what it means for a group to be disadvantaged, we will not appeal to any kind of historical considerations. Disadvantaged groups will simply implicitly be protected groups that have lower incomes than other groups\footnote{Since our condition does not explicitly use historical context, our argument will apply to \emph{all} ``disadvantaged'' groups, not just \emph{historically} disadvantaged ones}. We will argue that in this case, equalizing expected impact of false negatives and equalizing false negative rates are considerably different, and that equalizing expected impact is more fair to disadvantaged groups. We will call equalizing expected impact for financial situations \emph{equalizing financial impact}.
 
We also think equalizing financial impact may move existing notions of fairness closer to the kinds of notions that take into account historical injustices~\cite{Binns17}\cite{Parfit97}. Instead of taking into account historical injustices, we take into account current-day income. While this is not a perfect indicator of historical injustice, many historically disadvantaged groups have lower incomes, and so this perspective will be beneficial to them, as well as to other people will lower incomes. However, equalizing financial impact should not be seen as a \emph{substitute} for accounting for historical injustices in fairness, but just as a notion that aligns with this goal.

The motivation for our definition comes from the definition of equalized opportunity (Equation \ref{def:equal-opp}) from the work of Hardt, Price, and Srebro~\cite{HPS16}. Let us consider the case where our classifier is deciding something that can be plausibly tied to financial value, like giving out a loan. Fix protected groups $a_1$ and $a_2$. Equalized opportunity says that, out of all the people in group $a_1$ who are qualified for loans, the fraction of them that do \emph{not} receive loans should be the same as the fraction of people in group $a_2$ who are qualified for loans and do not receive loans. 

Now suppose $a_2$ is a group of people with lower income (say, median income \$25000 a year) and $a_1$ is a group of people with higher income (say, median income \$120000 a year). Then we know that $a_1$ and $a_2$ have very different marginal utilities of wealth, since it is a widely accepted phenomenon that the marginal utility of wealth decreases as a person becomes more wealthy\footnote{Different incomes can have very different meanings depending on location; we will ignore this for simplicity.}. It seems plausible that this phenomenon should extend to the marginal utility of receiving a loan. People who make less money probably have uses for loaned money that will bring them more value than the corresponding uses of a wealthier person. For example, someone who makes \$25000 a year may be able to use the loan to purchase a car, which they can use to expand their employment opportunities and/or decrease their cost of living by moving to a cheaper area. This change in lifestyle resulting from the loan can lead to a direct financial gain for them. A person who makes \$120000 may also have good uses for the loan, which give them some sort of financial gain. But this financial gain, unless it is an extremely large amount, will not bring them as much value as the loan made to the person with an income of \$25000 a year. Hence, it appears that decreasing marginal utility for wealth implies that loans have less marginal utility for wealthier people. This is particularly true if marginal utility of wealth decreases very quickly, and hence even if a wealthier person can utilize the loan to make somewhat more money than the poorer person, this increased amount of money will be dwarfed by the rapid decrease in the value of money according to this model. Indeed, according to some models, the utility of wealth can be modeled logarithmically, which implies that the marginal utility of wealth scales like $1/x$~\cite{EasKlein2010}.

From here on, we will assume that loans have decreasing marginal utility with respect to wealth, when restricted to people who can pay back the loan. But then, equalizing false negative rates between two classes $a_1$ and $a_2$ which have vastly different distributions of incomes is now less satisfactory. Even though the rates of mistakenly not giving a loan to qualified people in $a_1$ and not giving a loan to qualified people in $a_2$ are the same, these mistakes generally have very different effects on these people. So, as far as the \emph{impact} of these mistakes, the negative impact on the people in the poorer class is far greater. Hence, our goal will be to modify equalized opportunity to equalize the negative impact of these mistakes. 

\subsection{Formal Definition}
Let $\psi: \mathbb{R}^{+} \to \mathbb{R}^{+}$ map incomes of qualified people to the marginal utility of a loan of fixed size. For simplicity, we will assume that the marginal utility of a loan is purely a function of income, and so $\psi$ captures the marginal utility of a loan perfectly. For a brief discussion on how to generalize $\psi$ to a more realistic setting, see Section \ref{sec:psi}. 

\begin{definition}
\label{def:equalized-impact}
Let $B \subseteq X$ be the feature that corresponds to income. Now we are ready to define equalized financial impact. A classifier $\hat{Y}$ satisfies equalized financial impact if, for all $a_1, a_2 \in A$, we have that 
\begin{equation} 
\label{eq:equalized-impact-1}
\sum_{b \in B} \Pr[\hat{Y} = 0, B=b | A=a_1, Y=1] \cdot \psi(b) =  \sum_{b \in B} \Pr[\hat{Y} = 0, B=b | A=a_2, Y=1] \cdot \psi(b)
\end{equation}
\end{definition}

So all that we have to do to modify the condition for equalized opportunity to satisfy the definition of equalized financial impact is weight how much we care about a false negative on a person based on their income. However, this condition is a little annoying, as the space of incomes is continuous. In Section \ref{sec:achievability}, we argue that we can actually just partition people into a finite set of income classes, which will let us rewrite the condition in a way such that we can achieve it with techniques from~\cite{HPS16}. In order to do this, we will need to leverage a certain property of $\psi$: that it can be well-approximated by a finite histogram.

\subsubsection{Properties and Discussion of $\psi$}
\label{sec:psi}

In order to make our condition easier to enforce, we will make the following three assumptions about $\psi$:
\begin{enumerate}
\item $\psi$ is nonnegative and bounded above.
\item $\psi$ is monotonically decreasing.
\item $\psi$ is a smooth function.
\end{enumerate}
Now we argue that $\psi$ is well-approximated by a finite histogram, where the number of buckets in the histogram is inversely proportional to the accuracy of the estimation. Noting that Equation \ref{eq:equalized-impact-1} is scale invariant, we can scale $\psi$ such that $\sup_x \psi(x) = 1$, without loss of generality. The three conditions together tell us that the function $\psi$ is sufficiently nice such that the following holds. Fix any $\varepsilon > 0$ and any point $x \in \mathbb{R}$. Then there exists some $\delta > 0$ such that, for any $y \in (x - \delta, x + \delta)$, we have that $\psi(y) \in (\psi(x) - \varepsilon, \psi(x) + \varepsilon)$. The upshot of this is that $\psi$ can be approximated by a histogram. Fix $\varepsilon > 0$ to be a small constant. Then there is a finite set of points $Z = \{ z_1, \ldots, z_n\}$ (indeed, with $n = 1 / \varepsilon$) such that, for any $y \in \mathbb{R}$, there exists some $z_i \in Z$ such that $y \in (z_i, z_{i+1})$ and $\psi(y) \in (\psi(z_i) - \varepsilon, \psi(z_i))$. Hence, if we define $\widetilde{\psi}$ such that $\widetilde{\psi}(y) = \psi(z_i)$, then $\widetilde{\psi}$ never differs from $\psi$ by more than an additive factor of $\varepsilon$.

Before progressing to how we can utilize the framework of~\cite{HPS16} to achieve equalized financial impact, we briefly discuss the various assumptions we made on $\psi$, as well as how to generalize $\psi$ to settings where factors other than income affect the marginal utility of a loan.

The assumption that $\psi$ is bounded above may be debatable - perhaps one could suggest that a loan could have unbounded marginal value as we consider incomes going to 0. Not only does this seem unlikely to be true, but also this kind of unbounded $\psi$ leads to an overly strong constraint when applied to people who are unemployed\footnote{Perhaps this is an argument for the inclusion of other relevant factors into the domain of $\psi$.}, which might force a ``fair'' classifier to always give these people a loan. Hence, imposing an upper bound on $\psi$ seems reasonable. The second condition is simply that a loan should have marginal utility that decreases with wealth - we have already argued for this based on the decreasing marginal utility of wealth. The final condition arises naturally out of existing models of marginal utility versus wealth.

In reality, the marginal utility of a loan is not necessarily purely a function of income. For instance, it could also rely on other aspects of wealth, like assets. Without loss of generality, we will ignore this. These kinds of other features (financial or not) that are significant when gauging the affect of a loan can be factored into this framework in a similar way, by making $\psi$ into a multivariate function of all these features. Factors that affect the marginal utility of a loan in a relatively insignificant way can just be averaged over, so $\psi$ is just the expected marginal utility of a loan given that all its input variables are fixed.

\subsection{Achievability}
\label{sec:achievability}

We define the meta-feature $B'$ such that $B' = z_i \iff B = b \text{ and } \widetilde{\psi}(b) = \psi(z_i)$. Now, up to a small loss factor, we can replace $\psi$ with $\widetilde{\psi}$ and use the fact that $\psi(b) \approx \widetilde{\psi}(b) = \psi(z_i)$ to rewrite an approximate version of Equation \ref{eq:equalized-impact-1} as follows:
\begin{equation}
\label{eq:equalized-impact-2}
 \sum_{z_i \in Z} \Pr[\hat{Y} = 0, B' = z_i | A=a_1, Y=1] \cdot \psi(z_i) =  \sum_{z_i \in Z} \Pr[\hat{Y} = 0, B' = z_i | A=a_2, Y=1] \cdot \psi(z_i)
\end{equation}
This formulation is more convenient than Equation \ref{eq:equalized-impact-1} for the reason that the domain of $B'$ is finite (in fact, its size is roughly $1/\varepsilon$). Therefore, by considering $B'$ to be a ``protected feature,'' we can simply adapt the framework of~\cite{HPS16} to guarantee that a learned classifier satisfies Equation \ref{eq:equalized-impact-2} in much the same way that~\cite{HPS16} guarantee equalized opportunity (and equalized odds). Now we prove the claim that a classifier that satisfies Equation \ref{eq:equalized-impact-2} approximately satisfies Equation \ref{eq:equalized-impact-1}.

\begin{proposition}
\label{prop:approx}
A classifier $\hat{Y}$ that satisfies Equation \ref{eq:equalized-impact-2} also satisfies that
\[ \sum_{b \in B} \Pr[\hat{Y} = 0, B=b | A=a_1, Y=1] \cdot \psi(b) \stackrel{2 \varepsilon}{\approx}  \sum_{b \in B} \Pr[\hat{Y} = 0, B=b | A=a_2, Y=1] \cdot \psi(b) \]
where $x \stackrel{\varepsilon}{\approx} y$ if and only if $x \in (y - \varepsilon, y + \varepsilon)$.
\end{proposition}
\begin{proof}
We can rewrite the LHS of Equation \ref{eq:equalized-impact-2} as follows:
\begin{align*}
 \sum_{z_i \in Z} \Pr[\hat{Y} = 0, B' = z_i | A=a_1, Y=1] \cdot \psi(z_i) &=  \sum_{z_i \in Z} \sum_{b:\widetilde{\psi}(b) = \psi(z_i)}   \Pr[\hat{Y} = 0, B = b | A=a_1, Y=1] \cdot \psi(z_i) \\
&= \sum_{z_i \in Z} \sum_{b:\widetilde{\psi}(b) = \psi(z_i)} \Pr[\hat{Y} = 0, B = b | A=a_1, Y=1] \cdot \widetilde{\psi}(b) \\
&\stackrel{\varepsilon}{\approx} \sum_{z_i \in Z} \sum_{b:\widetilde{\psi}(b) = \psi(z_i)} \Pr[\hat{Y} = 0, B = b | A=a_1, Y=1] \cdot \psi(b) \\
&=  \sum_{b \in B} \Pr[\hat{Y} = 0, B = b | A=a_1, Y=1] \cdot \psi(b)
\end{align*}
The first line follows from the definition of $B'$, and the second line and third lines from the definition of $\widetilde{\psi}$. We can apply the same argument to the RHS of Equation \ref{eq:equalized-impact-2} to conclude the proof.
\end{proof}

We now sketch how to enforce that a classifier satisfies Equation \ref{eq:equalized-impact-2}, using the framework from~\cite{HPS16}. However, we do not give much background on~\cite{HPS16}, and refer the reader to their paper if interested. 

Roughly speaking, in~\cite{HPS16} the authors describe how to transform a given classifier $\hat{Y}$ into a classifier $\widetilde{Y}$ that satisfies equalized opportunity, such that the transformation is efficient and the accuracy remains similar in some cases. We can adapt this same transformation to guarantee equalized financial impact, by simply changing their equalized opportunity constraints to the equalized financial impact constraints. Note that our constraints use information about incomes - specifically, we use information about the meta-feature $B'$. Hence, our transformation will have to not only depend on the protected attributes $A$, but also on the meta-features $B'$: this is simply unavoidable since the condition we want to guarantee depends on the distribution of incomes. Hence, we might as well make $B'$ a protected feature in their framework, which is roughly what we do. Namely, we simply run the following optimization problem from~\cite{HPS16} to get $\widetilde{Y}$.

\begin{align*}
&min_{\widetilde{Y}} \hspace*{1cm} \mathbb{E}_{Y, \widetilde{Y}} \ell(\widetilde{Y}, Y) \\
&s.t. \text{ for all } (a,z) \in A \times B': \\
&\left(\Pr[\widetilde{Y}=1 | A=a_i, B' = z_j, Y=1]\right)_{i,j} \in \mathcal{P}(\hat{Y}) \\
&\text{ and } \forall a_1, a_2 \in A: \\
&\sum_{z_i \in Z} \Pr[\hat{Y} = 0, B' = z_i | A=a_1, Y=1] \cdot \psi(z_i) =  \sum_{z_i \in Z} \Pr[\hat{Y} = 0, B' = z_i | A=a_2, Y=1] \cdot \psi(z_i)
\end{align*} 
where $\left(\Pr[\widetilde{Y}=1 | A=a_i, B' = z_j, Y=1]\right)_{i,j}$ denotes a vector of size $|A| \cdot |B'|$ and the $(i,j)^{th}$ entry of the vector is $\Pr[\widetilde{Y}=1 | A=a_i, B' = z_j, Y=1]$. The first constraint ensures that the classifier $\widetilde{Y}$ only depends on the ground truth $Y$, the original classifier $\hat{Y}$, the protected variables $A$, and the income classes $B'$, as done in~\cite{HPS16}. The final constraint is just equalized financial impact. 

Since $|A|$ and $|B'|$ are both assumed to be finite, a finite number of samples suffices to estimate the distribution $\Pr[\hat{Y} | A, B', Y=1]$ well, which is sufficient to implement this procedure efficiently. The amount of labelled data we will need will be in terms of the sizes of $A$ and $B'$, so it is advantageous to have them be smaller from the standpoint of computational complexity. We can assume that we do not have too many protected groups $A$, so that will be small, and the size of $B'$ is $1/\varepsilon$. So the more accurately we want to equalize impact, the more samples and time we need. With regards to the accuracy of the classifier, if we assume that $\psi$ is constant, this linear program reduces to the one in ~\cite{HPS16}, and so we naturally get the same guarantee. We leave analysis of accuracy of the resulting classifier in other, more realistic, cases to future work.
\subsection{Discussion and Further Questions}
\label{sec:discussion}
At a high level, equalized financial impact stipulates that the false negatives that a classifier makes should have equal impact across all protected classes. We believe that this is a good notion of fairness in situations where the decision being made has financial impact, like in the case of deciding to whom one should give a loan. This is because it conforms to the spirit of equalized opportunity, in that all qualified individuals should be treated equally regardless of which protected class they are in, but defines ``equal treatment'' as saying that they should be equally harmed by false negatives on average, rather than saying that they should have the same false negative rate on average. We believe that the main takeaway message from this is that we should consider how much a classifier's mistakes hurt different groups. Sometimes mistakes in classification hurt the misclassified people in the same way, regardless of other factors. However, sometimes income (and other financial factors) are extremely relevant in gauging how much damage a misclassification does. This consideration can also extend to non-financial factors, as mentioned in Section \ref{sec:main-section}. Indirectly, this benefits many historically disadvantaged classes, that tend to have lower incomes in the present day. However, it does not directly take into account historical considerations.

We could also consider extending false positive fairness for loans to account for utilities just like we demonstrated for false negative fairness. This is justifiable for the same reasons that we outlined in Section \ref{sec:main-section}. The main obstruction is reasoning about the associated marginal utility function, which measures how much false positives hurt based on various attributes. For example, should it be increasing with income, decreasing with income, constant, or something else? Perhaps one can argue that false positives have abnormally bad effect on poorer people, as giving a loan that they cannot pay back can trap them in a cycle of debt. It is not clear to us how such considerations should be factored into the marginal disvalue of a false positive in general, but we think it is a very interesting and relevant question towards ensuring fairness for loans.

Even if this question is answered, equalizing both false positive impact and false negative impact could be hugely detrimental to the accuracy of the classifier\footnote{This could even be the case for just equalizing false negative \emph{or} false positive impacts for many cases}. Furthermore, it could be at odds with calibration and predictive parity just like equalized odds is~\cite{KMR16}~\cite{Choul17} (in fact, it trivially will be when the marginal utility function is constant for both false positives and negatives). Examining conditions of $\psi$ and the classifier $\hat{Y}$ for which these notions are at odds is an interesting question. Along this line of thought, we note that even using a crude approximation to the marginal utility function in in Equation \ref{eq:equalized-impact-2} can still yield ``more fair'' outcomes than equalizing error rates alone. Allowing for this kind of flexibility could alleviate tensions with competing goals of having predictive parity or a high-accuracy classifier.

Finally, we could consider implementing this framework to equalize the impact of classifier mistakes for any decision problem. However, being able to implement our framework relies on approximating the marginal utility function $\psi$, which has been studied for wealth, and therefore seems eminently extendable to false negative fairness for loans. To implement equalized impact more generally, we need to approximate $\psi$ for different problems. Furthermore, we would need $\psi$ to have a good histogram approximation, as we did in Section \ref{sec:psi}.

\section{Conclusion}
In this paper, we introduced a new notion of fairness in supervised learning that we called equalized financial impact. This notion is very similar to an existing notion in computer science literature called equalized opportunity defined by Hardt, Price, and Srebro~\cite{HPS16}, and focuses on equalizing harms across groups much like Altman et al.~\cite{AVW18}. Equalized financial impact applies to situations in which our classifier is making a decision that has direct financial value to the people in question, and therefore for which we can reason that false negatives are more harmful to people who make less money. In situations where false negatives are equally harmful to people who make more money as they are to people who make less money, equalized financial impact is identical to equalized opportunity. 

 Equalized financial impact only differs from equalized opportunity because of the decreasing marginal utility of wealth. However, in these situations, we have argued that it is a more relevant fairness condition than equalized opportunity. Generally, one can consider equalized impact in arbitrary settings, but one needs to have a good model of the marginal utility function $\psi$ in order to apply the framework we have described. In this sense, equalized odds / opportunity is a simpler condition to enforce than equalizing impact.

The notion of equalized financial impact we provided was partially motivated by an observation we found in the paper of Binns~\cite{Binns17} and Barocas and Selbst~\cite{BS16}, which say that in situations where historical injustices are relevant, definitions of fairness should factor in these historical injustices and not just classify based on qualifications in the present day. Equalized financial impact does not directly factor in historical injustices, but indirectly benefits historically disadvantaged communities by benefitting communities with lower incomes. Just like~\cite{Binns17} and~\cite{BS16}, we believe there is considerable work to be done in coming up with meaningful definitions of fairness that can reflect historical disadvantages. Still, we believe that our notion can meaningfully benefit disadvantaged communities by rectifying some injustices they currently face.
\section{Acknowledgements}
We would like to thank Arturs Backurs, Milo Phillips-Brown, Miriam Schoenfield, Aloni Cohen, and Ran Canetti for useful discussions and for reading over a previous draft of this paper. We would also like to thank David Grant for providing me with helpful resources to help me understand fairness literature from a philosophy perspective. 
\bibliographystyle{alpha}
\bibliography{bibliography}
\end{document}